\documentclass[10pt]{article}

\usepackage{amsmath}
\usepackage{amssymb}

\usepackage{graphicx}

\usepackage{cite}

\usepackage{color}

\usepackage[labelfont=bf,labelsep=period,justification=raggedright,singlelinecheck=false]{caption}

\usepackage{times}
\usepackage{epsfig}
\usepackage{graphicx,subfigure}
\usepackage{amsthm}
\usepackage{nccbbb}
\usepackage{galois}
\usepackage{url}
\usepackage{booktabs}
\usepackage{multirow}
\usepackage{bbm}
\usepackage{array}
\def\R{\bbbr}

\newtheorem{defi}{\hspace{1.5em}Definition}
\newtheorem{theorem}{\hspace{1.5em}Theorem}

\makeatletter
\renewcommand{\@biblabel}[1]{\quad#1.}
\makeatother

\date{}

\begin{document}

% Title must be 150 characters or less
\begin{flushleft}
{\Large
%\textbf{From Maximal to Densest: Density-Based Region Search with Arbitrary Shape}
\textbf{Density-Based Region Search with Arbitrary Shape for Object Localization}
}
% Insert Author names, affiliations and corresponding author email.
\\
%Ji Zhao$^{1}$,
Ji Zhao$^{1,\ast}$,
Deyu Meng$^{2}$,
Jiayi Ma$^{3}$
\\
\bf{1} %Haidian district, Beijing 100080, China
Samsung Advanced Institute of Technology, Beijing 100027, China
%\bf{1} School of Automation, Huazhong University of Science and Technology, Wuhan 430074, China
%Institute for Pattern Recognition and Artificial Intelligence
\\
\bf{2} School of Mathematics and Statistics, Xi'an Jiaotong University, Xi'an 710049, China
\\
\bf{3} Electronic Information School, Wuhan University, Wuhan 430072, China
%\bf{3} School of Automation, Huazhong University of Science and Technology, Wuhan 430074, China
\\
$\ast$ E-mail: zhaoji84@gmail.com
\end{flushleft}

% Please keep the abstract between 250 and 300 words
\section*{Abstract}
Region search is widely used for object localization. Typically, the region search methods project the score of a classifier into an image plane, and then search the region with the maximal score. The recently proposed region search methods, such as efficient subwindow search and efficient region search, %which localize objects from the score distribution on an image
are much more efficient than sliding window search. However, for some classifiers and tasks, the projected scores are nearly all positive, and hence maximizing the score of a region results in localizing nearly the entire images as objects, which is meaningless.

In this paper, we observe that the large scores are mainly concentrated on or around objects. Based on this observation, we propose a method, named level set maximum-weight connected subgraph (LS-MWCS), which localizes objects with arbitrary shapes by searching regions with the densest score rather than the maximal score. The region density can be controlled by a parameter flexibly. And we prove an important property of the proposed LS-MWCS, which guarantees that the region with the densest score can be searched. Moreover, the LS-MWCS can be efficiently optimized by belief propagation. The method is evaluated on the problem of weakly-supervised object localization, and the quantitative results demonstrate the superiorities of our LS-MWCS compared to other state-of-the-art methods.

% Please keep the Author Summary between 150 and 200 words
% Use first person. PLoS ONE authors please skip this step.
% Author Summary not valid for PLoS ONE submissions.
%\section*{Author Summary}

\section{Introduction}
In object localization, how to train the classifier and localize the objects is a chicken-and-egg problem \cite{Nguyen09}. Intuitively, if we know the locations of the objects, training a good classifier should be easier \cite{uijlings12}; alternatively, if we have an ideal classifier, the higher scores of the classifier should distribute on or surrounding the objects, and localization should be easier. Therefore, it is common to consider these two tasks jointly. %Among these two tasks,
Given the spatial distribution of the classifier's score,
finding the region that most probably containing object is known as region search problem \cite{lampert08,Vijayanarasimhan11}.  %Russakovsky12

If we want to localize objects in a weakly-supervised manner, i.e., the label of a training image is in the image level, the dependency between classifier training and object region searching will be even stronger.
The classifiers are typically trained by adopting the standard bag-of-words representation and support vector machines (SVMs), while the bottleneck is the region search. Here, region search aims to find candidate regions according to the spatial distribution of SVM scores.
The popular sliding-window-based search, which exhaustively applies the classifier to rectangles within an image, is a natural way for region search. However, it is not efficient. Many heuristics have been used to speed up the search procedure, but they also introduce the risk of imprecisely localizing the object or even missing it. Many methods with global optimum and efficient solutions have been proposed to instead sliding window, here we name a few as representative ones.

% % % % % % % % % %
% figure 1
% % % % % % % % % %
\begin{figure}[!t]
\begin{center}
{\includegraphics[width=0.95\linewidth]{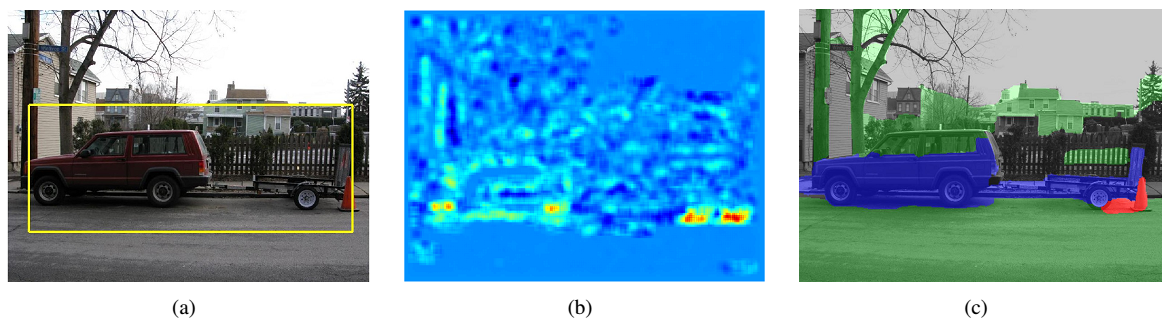}}
\end{center}
\vspace{-0.2in}
\caption{{\bf Detection by maximal score vs. detection by densest score}. (a) ESS \cite{lampert08} that maximizes classifier's score tends to find much larger area than objects. (b) Visualization of score distribution for classifier  ($\chi^2$-SVM) \cite{uijlings12}. (c) Localization of LS-MWCS. The green region denotes the detection with level $0$ (corresponds to ERS \cite{Vijayanarasimhan11}). By increasing the level, the localization shrinks from the green region to blue and red regions. Best viewed in color.}
\label{fig:demo}
\end{figure}

The efficient subwindow search (ESS) is a remedy for sliding-window search \cite{lampert08}.
ESS finds the rectangle that has the maximal score within an image. This turns out to be the 2D maximum subarray problem and the authors proposed an efficient branch-and-bound strategy to solve it. Inspired by the great success of ESS,
many methods have been proposed %to improve it.
%there are many works for improving it.
%, and this is the maximum subarray problem.
%In \cite{AnS11}, the additive property of score function in ESS was relaxed to submodularity.
%There are also some works
to generalize the rectangle restriction in ESS. For example, the rectangle was generalized  to composite boxes and polygons in \cite{Yeh09}.
%In \cite{ZhangZ10}, Zhang {\it et al.} forced the resulting shape to be well aligned with edges.
%Feature-centric efficient subwindow search is another remedy for sliding-window search \cite{Lehmann09}. It takes advantage of the scale of features, and integrates the scale to Hough transform-based object recognition. However, this method cannot be used for the popular bag-of-words feature representation.
The advantages of ESS lie in that it is parameter-free, and there exist efficient algorithms to solve it.
It also has certain limitations.
%The limitations of ESS are also obvious.
(i) If all projected scores are positive, ESS will find the entire image to be the best subwindow. (ii) ESS can only find rectangular bounding boxes. For objects of non-rectangle shapes, noises in windows may mislead the localization, causing unsatisfactory results.

In order to find regions with arbitrary shapes, efficient region search (ERS) \cite{Vijayanarasimhan11} was proposed for object localization. ERS converts region search into the maximum-weight connected subgraph (MWCS) problem, and it can be equivalently transformed into an instance of the prize-collecting Steiner tree (PCST) problem \cite{Dittrich08}. PCST is an NP-Complete problem but has efficient approximation solutions \cite{Ideker02,Bailly10}.
The localization ability of ERS has been demonstrated in supervised classification framework.
In \cite{ZhouC12}, an adaptive image grids scheme was proposed to improve ERS by finding a better segmentation.
ERS is successful for supervised localization. However, if all projected scores are positive, it will also find the entire image to be the best region as that in ESS.

Besides ESS and ERS, some region search methods are proposed for interactive segmentation. For example,
twisted window search (TWS) \cite{GuS12a} and shape from point features (SFF) \cite{GuS12b} are proposed for interactive shape localization. SFF use $\alpha$-complex \cite{Edelsbrunner83} to build a filtration of simplicial complexes from a user-provided set of features. The densest complex connected component is viewed as the shape.
SFF is successful for interactive shape localization. However, it relies on the distinct weights manually labeled. Empirically, it is not suitable for object localization because the projected weights from the classifier are usually very noisy compared to human interaction.

The weakly-supervised object localization has higher demand for region search than supervised localization, since the distribution of projected scores is in general very noisy.
%Hoai {\it et al.} proposed a framework for weakly-supervised object localization that simultaneously finds the most discriminative subwindows in the images and learns an effective classifier \cite{Nguyen09}. They used ESS in the iteration process for region search. However, as discovered by several independent works, ESS tends to find the entire image as subwindow \cite{Siva11, Russakovsky12}.
The basic idea of our method is based on the following observation: although the spatial distribution of a classifier's scores seems to be random, the densest area of the score distribution is on or near the object, see Fig.~\ref{fig:demo}(b) for example.
Instead of finding the maximal score as done in ESS and ERS, we attempt to search the regions with dense scores. The comparison of localization by densest score and maximal score is demonstrated in  Fig.~\ref{fig:demo}.
Our region search method is called {\bf LS-MWCS (level set maximum-weight connected subgraph)}, which has the following two advantages: (i) it is able to localize regions of arbitrary shape; (ii) it works well when all the weights are positive.

The contribution of this paper is two-fold. First, we propose a region search method that can be used for very noisy weight distributions, such as for weakly-supervised localization.
Second, we validate the effectiveness of the densest score for region search, which in previous literatures is dominantly achieved by maximizing the score.

\subsection{Related Work}
Region search is a key technology in weakly supervised localization (WSL). WSL is usually modeled as multiple instance learning (MIL). In the MIL setting, each image is modeled as
a bag of regions, and each region is an instance. With two classes, the negative
bag only contains negative instances and the positive consists of at least one positive. The
goal of MIL is to label the positive instances in the positive bags. Region search corresponds to finding the region (instance) in the positive image (bag) that triggers the positive label. In the past few years, many MIL
algorithms have been successfully used for weakly supervised learning, such as
MILboost \cite{Galleguillos08} and MI-SVM \cite{Nguyen09}. In \cite{ZhaoJ14}, a region weighting method is proposed for WSL, which is customized for bag-of-words feature representation and non-linear SVM classifiers.
Region search has a close relationship with common pattern discovery from images that share common contents, such as co-segmentation and image feature matching \cite{Matas02,MaJ14,ma2014mixture}.

\section{Density-Based Region Search}
There are two schemes for region search: one is based on feature points \cite{lampert08}, and the other is based on super-pixels from over-segmentation \cite{Vijayanarasimhan11}.
For methods based on feature points, it is difficult to localize regions that are consistent with the object boundaries. By contrast, localization methods based on super-pixels can guarantee the local consistency. In this paper, we focus on region search with super-pixels.

The effectiveness of ERS has been validated for supervised localization, while its performance for weakly-supervised localization has not been reported.
Empirically, we found that ERS has the following problem when it is directly used in weakly-supervised localization. The shape and area of the connected component found by ERS cannot be directly controlled. In weakly supervised localization, since the positive weights are scattered over the entire image, ERS tends to find spindly shapes that cover nearly the whole image.
However, we find that the densest area of the score distribution is on or near the object, see Fig.~\ref{fig:demo}. Based on this observation, in this section we will extend ERS to find regions with densest scores.

Fist, we give some formal definition.
The essence of ERS is the
maximum-weight connected subgraph (MWCS) problem, which is defined as follows.

\begin{defi}[{\bf Maximum-Weight Connected Subgraph (MWCS) Problem \cite{Dittrich08, Vijayanarasimhan11}}]\label{mwcs}
Given a connected undirected, vertex-weighted graph $G = (V, E)$ with weights $\omega: V \rightarrow \R$, find a connected subgraph $T = (V_T \subseteq V, E_T \subseteq E)$ of $G$ to maximize the score $W(T) = \sum_{v \in V_T} \omega (v)$.
\end{defi}

We also define the density of a graph. In this paper, density means the average weight of a graph.

\begin{defi}[{\bf Density of Graph}]\label{density}
Given a connected undirected, vertex-weighted graph $G = (V, E)$ with weights $\omega: V \rightarrow \R$, the density is defined as $\frac{\sum_{v \in V} \omega (v)}{|V|}$, where$|\cdot|$ means the node number of a graph.
\end{defi}

In the region search problem, the set of vertices $V$ consists of the superpixels, and an edge in $E$ connects a pair of superpixels that share a boundary. The weight of a vertex $\omega(V)$ is the superpixel's classifier score. The basic idea of ERS \cite{Vijayanarasimhan11} is that the subgraph with maximal score corresponds to the localized object.

Since the MWCS problem aims to find a subgraph with maximum-weight, the density of the resulting subgraph cannot be controlled. In order to control the density of a subgraph, we define the level set MWCS (LS-MWCS) based on the definition of MWCS. As will be shown in Theorem 1, it has a close relationship with the density of the resulting subgraph.

\begin{defi}[{\bf Level Set Maximum-Weight Connected Subgraph (LS-MWCS) Problem}] \label{lsmwcs}
Given a connected undirected, vertex-weighted graph $G = (V, E)$ with weights $\omega: V \rightarrow \R$, the MWCS with level $\alpha$ aims to find a connected subgraph $T = (V_T \subseteq V, E_T \subseteq E)$ of $G$, that maximizes the score $W(T; \alpha) = \sum_{v \in V_T} (\omega (v) - \alpha)$.
\end{defi}

Based on the definition of LS-MWCS problem, we find the relationship between level $\alpha$ and the density of resulting subgraph. This property turns out to be very useful for density-based region search.

\begin{theorem}\label{theorem1}
Denote the MWCS of a graph with level $\alpha_1$ and $\alpha_2$ as $V^1_T$ and $V^2_T$ respectively. Suppose $0 \le \alpha_1 < \alpha_2$, then
(i) the node number satisfies $|V^1_T| \ge |V^2_T|$;
(ii) the weight summation satisfies $\sum_{v \in V^1_T} \omega(v) \ge \sum_{v \in V^2_T} \omega(v)$; %, where $|\cdot|$ means the node number of a graph.
(iii) if $|V^2_T| > 0$, then graph density satisfies $\frac{\sum_{v \in V^1_T} \omega(v)}{|V^1_T|} \le \frac{\sum_{v \in V^2_T} \omega(v)}{|V^2_T|}$.
\end{theorem}

\begin{proof} (i) Since $V^1_T$ and $V^2_T$ are the MWCS with level $\alpha_1$ and $\alpha_2$, by Definition~\ref{lsmwcs} we have
\begin{align}
\sum_{v \in V^1_T} (\omega(v)-\alpha_1) - \sum_{v \in V^2_T} (\omega(v)-\alpha_1) \ge 0; \label{equ:mvcs}\\
\sum_{v \in V^2_T} (\omega(v)-\alpha_2) - \sum_{v \in V^1_T} (\omega(v)-\alpha_2) \ge 0. \label{equ:mvcs2}
\end{align}
By adding the above two inequalities, we get $(\alpha_2 - \alpha_1) (|V^1_T| - |V^2_T|) \ge 0$. Since $\alpha_2 > \alpha_1$, the inequality $|V^1_T| \ge |V^2_T|$ holds.

(ii) Reformulate inequality~\eqref{equ:mvcs} as $\sum_{v \in V^1_T} \omega(v) - \sum_{v \in V^2_T} \omega(v) \ge \alpha_1 (|V^1_T| - |V^2_T|)$. Since we have proved that $|V^1_T| \ge |V^2_T|$, it is obvious that $\sum_{v \in V^1_T} \omega(v) \ge \sum_{v \in V^2_T} \omega(v)$ holds.

(iii) According to inequality~\eqref{equ:mvcs2}, we have
\begin{align}
\sum_{v \in V^2_T} \omega(v) - \alpha_2 |V^2_T| \ge \sum_{v \in V^1_T} \omega(v) - \alpha_2 |V^1_T|.
\end{align}
We have proved that $|V^1_T| \ge |V^2_T|$ and we suppose $V^2_T >0$, so the inequality still holds when we divide the left expression with $|V^2_T|$ and divide the right expression with $|V^1_T|$.
\end{proof}

% % % % % % % % % %
% figure
% % % % % % % % % %
\begin{figure}[!t]
\begin{center}
{\includegraphics[width=0.85\linewidth]{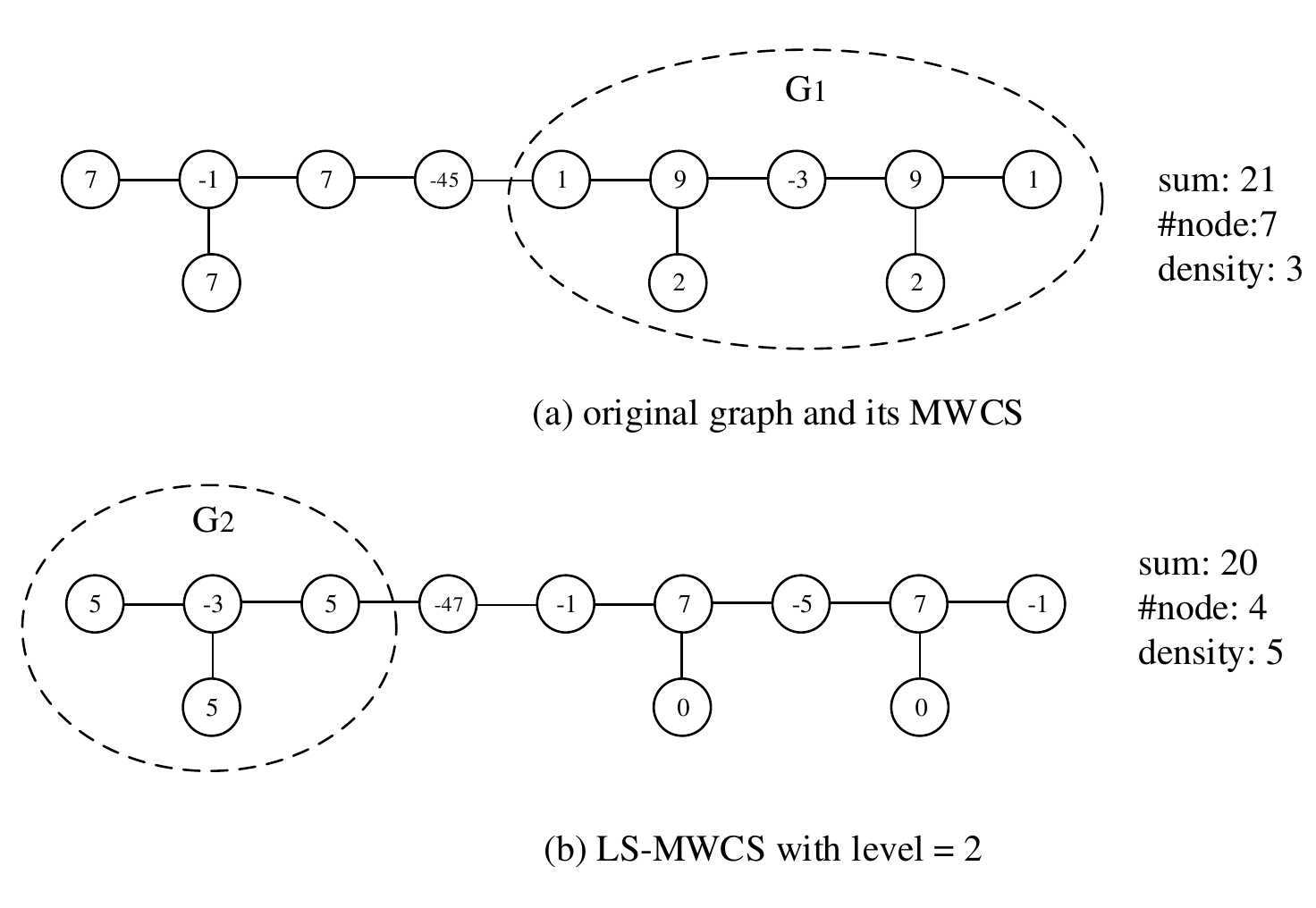}}
\end{center}
\vspace{-0.2in}
\caption{\bf An example of LS-MWCS.}
\label{fig:illumi}
\end{figure}

According to Theorem~\ref{theorem1}, when we increase the level $\alpha$, the node number and weight summation of the resulting subgraph is monotonically decreasing and the density of the subgraph is monotonically increasing.
Fig.~\ref{fig:illumi} gives an example. Fig.~\ref{fig:illumi}(a) is a graph with $12$ vertices and $11$ edges, the weights are labeled in corresponding nodes. According to Definition~\ref{mwcs}, it is easy to verify that its MWCS is subgraph $G_1$, where the node number, weight summation and density are $7$, $21$ and $3$ respectively.
If the weights of all the nodes minus $2$, we get a new graph as shown in Fig.~\ref{fig:illumi}(b). The MWCS of this new graph is subgraph $G_2$ . According to Definition~\ref{lsmwcs}, the LS-MWCS with level $2$ is constructed as following: the node and edge are the same as the subgraph $G_2$ and the corresponding node weights are derived from the original graph. So the node number, weight summation and density of LS-MWCS are $4$, $20$ and $5$ respectively.

Theorem~\ref{theorem1} provides a theoretical background for finding dense subgraph.
An interesting discovery in Theorem~\ref{theorem1} is that the weight summation of the generated subgraph is monotonically decreasing with the level $\alpha$.
This is useful in practice. For example, if we want to find the dense subgraph with predefined weight summation, we can use the binary search to find the proper level $\alpha$.

When the level $\alpha$ in LS-MWCS is determined, this problem can be easily converted to the standard MWCS problem. The MWCS problem is NP-complete. Fortunately, it can be transformed into an instance of the prize-collecting Steiner tree
(PCST) problem, as shown in \cite{Dittrich08}, which
can be efficiently solved in practice for our problem setting ($10^1 \sim 10^2$ nodes, $10^2 \sim 10^3$ edges).
In this paper we adopt an efficient belief propagation scheme to solve this problem\footnote{We use the code provided by the authors which is available from http://areeweb.polito.it/ricerca/cmp/code.} \cite{Bailly10}.

\subsection{Parameter Setting for LS-MWCS}

There is only one parameter, the level $\alpha$, for the problem of LS-MWCS. The level controls the density of the generated subgraph. By increasing $\alpha$, we can obtain the subgraph which has denser weights and vice versa.
For LS-MWCS based region search in the paper, this parameter is very useful when the projected classifier scores are all positive, such as Hough transform vote \cite{Maji09} and matrix completion multi-label classifiers \cite{CabralDCB14}.

Generally, the selection of level $\alpha$ is task-dependent. Here we provide three strategies for the selection of $\alpha$. First, the level $\alpha$ can be determined by a predefined node number, weight summation or graph density.
Second, we can find an optimal level that has best performance for a specific task that depends on level $\alpha$. Take region search for weakly-supervised localization as an example, %taking the scheme of iterated classifier training and region search as an example,
the optimal level $\alpha$ can be set as the number producing the highest image classification accuracy by cross-validation. Third, we can find $\alpha$ that corresponds to stable subgraphs. Here the notation of stable is similar to that in the definition of maximally stable extremal regions (MSER) \cite{Matas02}. A connect subgraph is stable if its weight summation does not change too much when the level $\alpha$ is slightly changed.
The efficient solution of stable subgraph is based on the fact that the weight summation is a monotonically decreasing function with level $\alpha$, as proved in Theorem~1.
%By increasing $\alpha$ from $0$, we can change the weight summation. As it has been proved in Theorem~1, the weight summation is a monotone decreasing function with level $\alpha$. %For some tasks, we may interested in the level $\alpha$ that generates a stable subgraph.

\section{Experimental Results}
In this section, we present the experiments of our density-based region search, including experiments on synthetic data and image datasets.

\subsection{Results on Synthesis Data}
Fig.~\ref{fig:mwcs} demonstrates an example of localizing regions using LS-MWCS.
The image is $19 \times 19$ regular grids. The nodes of the graph are the $361$ grids; the adjacency of the graph is determined by $4$-connected neighborhood in the image plane; the weight of each node is the blueness of each block.
The weights of the nodes are 2D Gaussian distribution with noise.
It means that all the nodes have positive weights. The solution of MWCS will find all the nodes as subgraph, as shown in Fig.~\ref{fig:mwcs}(a).
By increasing the level $\alpha$ in LS-MWCS, the optimal subgraph shrinks to denser subgraphs, see Fig.~\ref{fig:mwcs}(b) - (e).

% % % % % % % % % %
% figure 2
% % % % % % % % % %
\begin{figure*}[!t]
\begin{center}
{\includegraphics[width=0.98\linewidth]{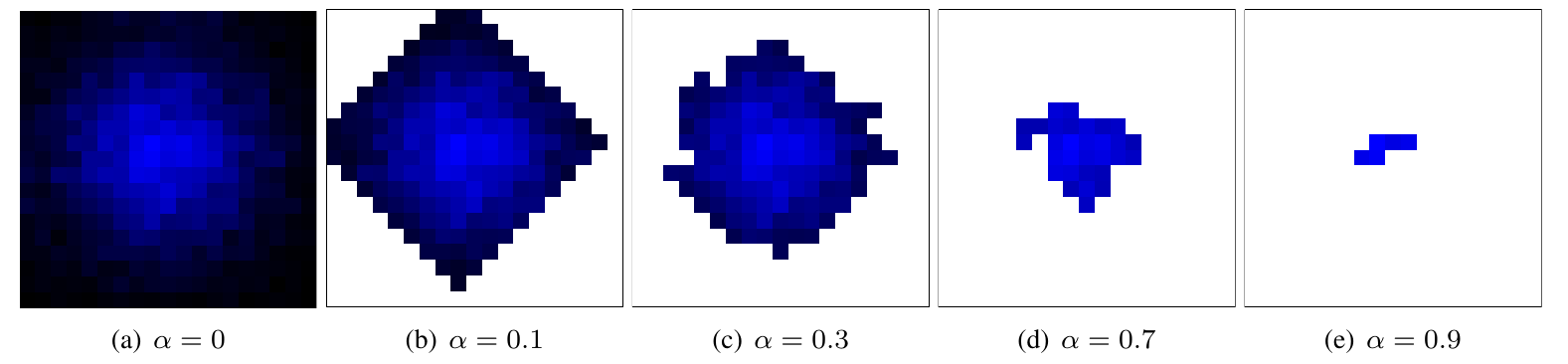}}
\end{center}
\caption{{\bf Localization by MWCS with different level $\alpha$.} The image is $19 \times 19$ regular grids. The nodes of the graph are the grids; the adjacency of the graph is determined by $4$-connected neighborhood in image plane; the weight of each node is the blueness of each block.}
\label{fig:mwcs}
\end{figure*}

\subsection{Object Localization on Image Dataset}
Next, we demonstrate the effectiveness of our method for weakly-supervised object localization. Similar to \cite{Nguyen09}, we use a scheme of iterated optimization for classifier construction and region search.

The dataset used in this experiment is Pittsburgh Car dataset \cite{Nguyen09}, which contains $400$ images, including $200$ positive samples and $200$ negative samples. There is only one car in each positive sample.
Half of the positive and negative samples are used as training data, and the rest are used for testing.

The bag-of-words implementation is provided by \textsc{VLFeat} toolbox \cite{vedaldi10}. It extracts dense SIFT feature uniformly \cite{lowe04}, and generates the visual words by vector quantization.
In all experiments, we extract $128$ dimensional SIFT descriptors over a grid with a step of $5$ pixels on fixed scale ($16 \times 16$ pixels).
The dictionary of $1000$ visual words is obtained by clustering $100,000$ samples via K-means clustering.
Linear SVM is adopted as classifier.
The SVM parameter and the level $\alpha$ in LS-MWCS are chosen via cross-validation. The optimal $\alpha$ is the one that maximizes the image classification accuracy.

Fig.~\ref{fig:RsltComp} demonstrates the qualitative performance comparison of $6$ methods on Pittsburgh Car dataset. We consider methods that only use the scores of features.
The 1st row in Fig.~\ref{fig:RsltComp} is the visualization of feature weights\footnote{We blur the weight using a Gaussian filter and then use the visualization code of O. Woodford available at http://www.robots.ox.ac.uk/~ojw/software.htm.}.
The 2nd to 5th rows are results of ESS \cite{lampert08}, TWS \cite{GuS12a}, SFF \cite{GuS12b} and ERS \cite{Vijayanarasimhan11}, respectively.
The 6th row is the results of our LS-MWCS. Clearly, our results are more visually pleasurable.

% % % % % % % % % %
% figure 3
% % % % % % % % % %
\begin{figure*}[!t]
\begin{center}
{\includegraphics[width=0.85\linewidth]{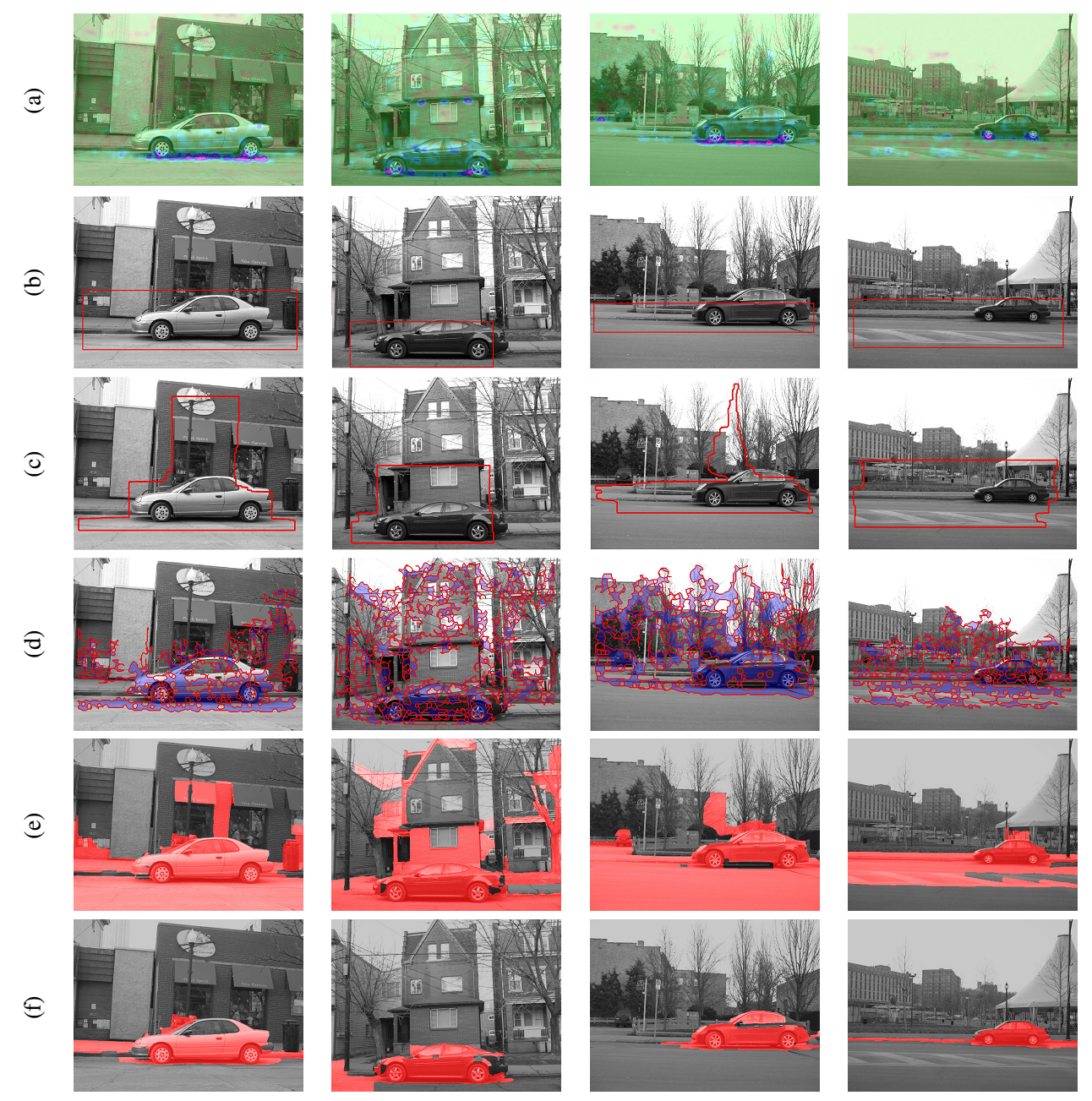}}
%%\end{tabular}
\end{center}
\caption{{\bf Performance comparison of region search methods.} (a) Visualization of BoW; (b) ESS; (c) TWS with parameter $\lambda = 10^{-3}$; (d) SFF; (e) ERS; (f) LS-WMCS.}
\label{fig:RsltComp}
\end{figure*}

From Fig.~\ref{fig:RsltComp}, we can see that ESS is biased toward choosing the entire image as the object of interest, and it is not capable of localizing objects accurately. % as discovered by \cite{Siva11}
In previous literatures, the effectiveness of TWS and SFF are validated in interactive shape localization, where the features are pixel-level and the weight of feature is discretized to only $3$ levels. Although TWS and SFF have potential applications in object localization, our experiments demonstrate that they cannot be used  for such application directly and some modifications are needed.

% % % % % % % % % %
% figure 4
% % % % % % % % % %
\begin{figure}[!t]
\begin{center}
  {\includegraphics[width=0.55\linewidth]{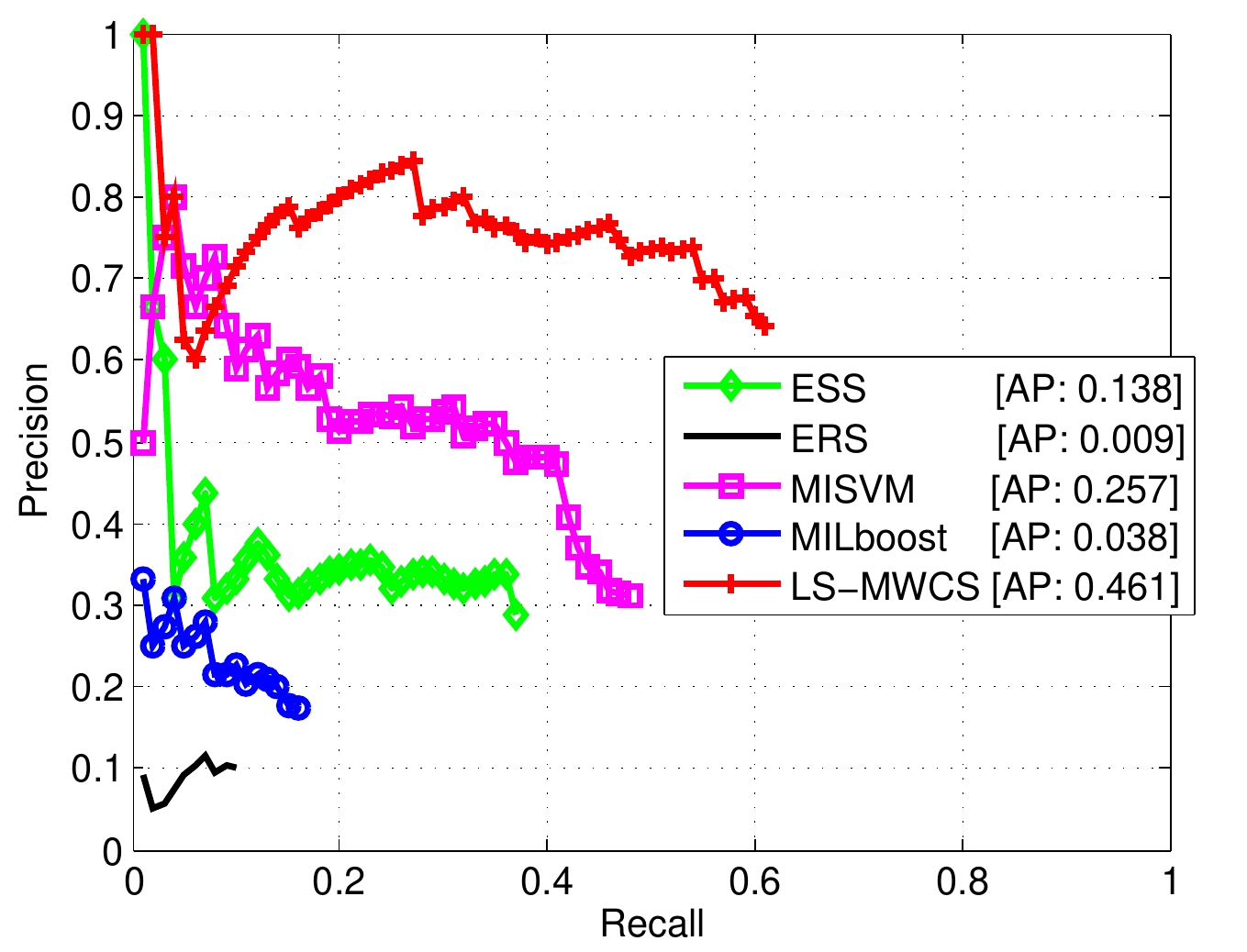}}
\caption{{\bf Localization performance comparison on Pittsburgh Car dataset.}}
\label{fig:localization}
\end{center}
\end{figure}

To provide a quantitative measure for the localization performance, we compared methods using precision-recall curves, as shown in Fig.~\ref{fig:localization}. We used the
area of overlap (AO) measure to evaluate the correctness of localization. For this
criterion, a threshold $t$ should be defined for AO to imply a correct detection.
In this paper, we set the threshold $t$ to $0.4$.
The precision-recall curves are shown in Fig.~\ref{fig:localization}. We compared our density-based region search method with other region search methods and weakly-supervised localization methods, including ESS \cite{lampert08}, ERS \cite{Vijayanarasimhan11}, MILboost \cite{Galleguillos08} and MI-SVM \cite{Nguyen09}. The average precisions (AP) of these methods are given in the legends. We can see that our LS-MWCS method performs significantly better than the baselines.

In our density-based region search, parameter $\alpha$ is the only parameter and is very important.
%As it has previously mentioned, we use cross-validation to find it.
If we set $\alpha$ as $0$, our density-based region search is the same as ERS. From Fig.~\ref{fig:RsltComp} and \ref{fig:localization}, we can see that our region search with level $\alpha$ performs better than original ERS in our experiment. This demonstrates that the control of score density by adjusting level $\alpha$ is effective for weakly-supervised localization.

\section{Conclusion}
In this paper, we studied the problem of region search for object localization. We proposed a new strategy of finding the regions with densest score rather than the conventional maximal score. Based on this strategy, a density-based region search method was then developed, which is called LS-MWCS (level set maximum-weight connected subgraph). The effectiveness of our method was validated by synthesis data and weakly-supervised localization for images. Our method can be easily extended to videos.

% Do NOT remove this, even if you are not including acknowledgments
\section*{Acknowledgments}
We are very grateful to Fernando De la Torre at Carnegie Mellon University for helpful discussions.
We thank Minh Hoai at The University of Oxford for providing the Pittsburgh Car dataset.

%\section*{References}
% The bibtex filename
%{
%\bibliography{egbib}

%}
%\section*{Figure Legends}
%\begin{figure}[!ht]
%\begin{center}
%%\includegraphics[width=4in]{figure_name.2.eps}
%\end{center}
%\caption{
%{\bf Bold the first sentence.}  Rest of figure 2  caption.  Caption
%should be left justified, as specified by the options to the caption
%package.
%}
%\label{Figure_label}
%\end{figure}

%\section*{Tables}
%\begin{table}[!ht]
%\caption{
%\bf{Table title}}
%\begin{tabular}{|c|c|c|}
%table information
%\end{tabular}
%\begin{flushleft}Table caption
%\end{flushleft}
%\label{tab:label}
% \end{table}

%J.Z. conceived the project and proved the theorem. D.M. designed the experiments and wrote the manuscript. J.M. performed the experiments. All authors reviewed the manuscript.

\end{document}